\documentclass[a4paper]{article}

\usepackage{microtype}

\usepackage[breaklinks,colorlinks=true,allcolors=black]{hyperref}
\usepackage{natbib}

\usepackage{xcolor}
\usepackage[algo2e,ruled,vlined,linesnumbered]{algorithm2e}
\usepackage{nicefrac}
\usepackage{amsmath}

\DeclareMathOperator{\act}{act}
\DeclareMathOperator{\cost}{cost}

\DeclareMathOperator{\OPT}{OPT}
\DeclareMathOperator{\DYN}{DYN}
\DeclareMathOperator{\DYNA}{DYN_{AGKP}}
\DeclareMathOperator{\DYNAp}{DYN_{AGKP}^{\act}}
\DeclareMathOperator{\Geo}{Geo}
\DeclareMathOperator{\Var}{Var}
\newcommand{\E}{\mathbb{E}}
\newcommand{\OddExponent}{\ensuremath{\textsc{OddExponent}}\xspace}
\newcommand{\share}{\textsc{Share}\xspace}

\usepackage{amssymb}
\usepackage{mathtools}
\usepackage{amsthm}

\usepackage[capitalize,noabbrev]{cleveref}

\theoremstyle{plain}
\newtheorem{theorem}{Theorem}[section]
\newtheorem{proposition}[theorem]{Proposition}
\newtheorem{lemma}[theorem]{Lemma}

\theoremstyle{definition}
\newtheorem{definition}[theorem]{Definition}

\newtheorem{observation}[theorem]{Observation}
\theoremstyle{remark}

\usepackage[textsize=tiny]{todonotes}

\begin{document}

\title{Mixing predictions for online metric algorithms}

\author{
Antonios Antoniadis\\
{\small University of Twente}\\
\and
Christian Coester\\
{\small University of Oxford}
\and
Marek Eli\'a\v{s}\\
{\small Bocconi University}
\and
Adam Polak\\
{\small Max Planck Institute for Informatics}
\and
Bertrand Simon\\
{\small IN2P3 Computing Center / CNRS}
}

\date{}

\maketitle

\begin{abstract}
A major technique in learning-augmented online algorithms is combining multiple algorithms or predictors.
Since the performance of each predictor may vary over time, it is desirable to use not the single best predictor as a benchmark, but rather a dynamic combination which follows different predictors at different times.
We design algorithms that combine predictions and are competitive against such dynamic combinations for a wide class of online problems, namely, metrical task systems.
Against the best (in hindsight) unconstrained combination of $\ell$ predictors, we obtain a competitive ratio of $O(\ell^2)$, and show that this is best possible.
However, for a benchmark with slightly constrained number of switches between different predictors, we can get a $(1+\epsilon)$-competitive algorithm.
Moreover, our algorithm can be adapted to access predictors in a bandit-like fashion, querying only one predictor at a time. An unexpected implication of one of our lower bounds is a new structural insight about covering formulations for the $k$-server problem.
\end{abstract}

\section{Introduction}

Motivated by the power of machine-learned predictions, the field of learning-augmented algorithms has been growing rapidly in recent years. In the classical field of online algorithms, an input sequence is revealed to an algorithm over time and it is assumed that at all times, no information about the future part of the input is available. In contrast, a learning-augmented algorithm additionally has access to \emph{predictions} (e.g., machine-learned) related to the future input. These predictions may be inaccurate, so a challenge is to simultaneously utilize high-quality predictions to their best advantage while at the same time avoiding to be misled by erroneous predictions.

An important technique in the field of learning-augmented algorithms is the method of combining multiple algorithms into a single hybrid algorithm that leverages the advantages of all individual algorithms. The basic idea goes back to several decades before the area of learning-augmented algorithms was born and also has applications, for example, in pure online algorithms: \citet{FiatRR90} defined a MIN operator on algorithms for the $k$-server problem that combines several algorithms into one whose cost matches the best of them up to a constant factor, and they used this technique to obtain the first competitive algorithm for the $k$-server problem.

In learning-augmented algorithms, similar combination techniques are employed for several purposes. Firstly, they are frequently used to make algorithms robust against prediction errors by combining an algorithm that mostly follows predictions with a classical online algorithm that ignores predictions (see, e.g., \cite{LykourisV21,PurohitSK18,Rohatgi20,AntoniadisCE0S20,Wei20,BamasMRS20,BansalCKPV22}). In fact, one might argue that almost all algorithms that utilize predictions while being robust to their error are at least implicitly a kind of combination of two algorithms. A second purpose, as employed in \cite{AntoniadisCEPS21}, is to combine several differently parameterized versions of the same algorithm in order to perform nearly as well as the version with the best parameter choice. These aforementioned works, like the majority of research in learning-augmented algorithms, focus on settings where a single predictor provides suggestions to the algorithm.

However, a third and perhaps the most relevant application of combining several algorithms in the learning-augmented realm is to be able to deal with multiple predictors. In practice it is often the case that several predictors are available, but they produce potentially conflicting advice; for example, there may be different ML models based on different methods or tailored to specific scenarios, or several human experts with contrary opinions. Since it is not clear a priori which of the predictors will be most reliable for the instance at hand, this creates the complication of deciding how to choose between the predictors. Research on learning-augmented algorithms with multiple predictions was initiated by \citet{GollapudiP19} for the ski rental problem, and subsequently also considered for additional problems such as multi-shop ski rental \cite{WangLW20}, facility location \cite{AlmanzaCLPR21}, matching, load balancing and non-clairvoyant scheduling \cite{DinitzILMV22}. For the objective of regret minimization, the case of multiple predictions was studied for online linear optimization \cite{BhaskaraC0P20} as well as caching \cite{EmekKS21}. Recently, \citet{Anand0KP22} designed a generic framework for online covering problems with multiple predictions, which they successfully applied to the problems of set cover, weighted caching and facility location.

A particularly interesting aspect of the work of \citet{Anand0KP22} is that the performance of their algorithm is comparable not only to the best individual predictor, but even to the best dynamic combination of predictors.  
This property is especially valuable on instances which are composed of several parts with different properties, as some predictor may be of high quality for certain sections of the input, but inferior to other predictors otherwise.

This raises the question of whether similar guarantees are also achievable for other problems. \citet{Anand0KP22} mention the $k$-server problem as a specific problem for which this would be interesting.

Here, our goal is to obtain generic methods for combining multiple predictors/algorithms applicable to a wide range of problems. To this end, we consider the class of \emph{metrical task systems (MTS)}. MTS was introduced by \cite{BorodinLS92} as a wide class of online problems, containing as special cases many other fundamental online problems such as $k$-server, caching, convex body/function chasing, layered graph traversal, dynamic power management etc. Thus, our results obtained for MTS directly translate to all these other problems as well.

We study this class of problems in a variety of settings: Against the best (in hindsight) dynamic combination of $\ell$ predictors, we obtain a competitive ratio of $O(\ell^2)$ and show that this is the best possible. This follows, essentially, from a reduction to the layered graph traversal problem. The aforementioned result allows the benchmark offline combination to switch between the $\ell$ predictors arbitrarily often. However, for more structured instances (for example, imagine an instance composed of blocks with different patterns, and different predictors specialized on these patterns), it is reasonable to assume that an optimal combination of predictors would not switch between them too often. We therefore consider the question whether a better performance is achievable under such an assumption. Indeed, against a dynamic combination benchmark that switches between the different predictors a moderately limited number of times, we achieve a $(1+\epsilon)$-competitive algorithm.

Since querying predictors may be costly \cite{Im0PP22,EmekKS21}, we also consider a setting where the learning-augmented algorithm can consult only one predictor per time step (similar to multi-armed bandits). We show that very similar guarantees can be achieved also for this setting.

\subsection{Preliminaries}
\paragraph{Metrical Task Systems.} In Metrical Task Systems (MTS), we are given a metric space $(M,d)$, whose points are called \emph{states}. An algorithm starts in some initial state $s_0\in M$. At each time $t=1,2,\dots,T$, a task appears, specified by some cost function $c_t\colon M\to\mathbb R_{\ge 0}\cup\{\infty\}$ that assigns to each state the cost of serving the task in that state. In response, the algorithm chooses a state $s_t\in M$, paying \emph{movement cost} $d(s_{t-1},s_t)$ and \emph{service cost} $c_t(s_t)$. We emphasize that $s_t$ can be chosen \emph{after} $c_t$ is known, but before $c_{t+1}$ is revealed.

\paragraph{$\boldsymbol{k}$-server.} In the $k$-server problem, we are given a metric space $(M,d)$, and $k$ servers are located at points of $M$. At each time $t=1,2,\dots,T$, a point $r_t\in M$ is requested, and an algorithm must choose one of the servers to move to $r_t$. The cost is the total distance travelled by servers. Note that $k$-server is an MTS in the metric space of server configurations.\footnote{I.e., $k$-server in $M$ can be cast as an MTS by taking the set of $k$-server configurations (i.e., size-$k$-subsets of $M$) as the metric space for MTS. Then $c_t$ assigns cost $0$ to configurations containing $r_t$ and cost $\infty$ to other configurations.}

\paragraph{Competitive ratio.} An algorithm $A$ for an online minimization problem is called $\rho$-competitive if
\begin{align}
\cost(A)\le \rho\cdot\OPT+c\label{eq:comp}
\end{align}
for every instance of the problem, where $\cost(A)$ is the cost of $A$ on the instance (or the expected cost, if $A$ is randomized), $\OPT$ is the optimal (offline) cost, and $c$ is a constant independent of the input sequence. If we replace $\OPT$ by some other benchmark $B$, we also say that $A$ is \emph{$\rho$-competitive against $B$}. The minimal $\rho$ satisfying \eqref{eq:comp} is also called the \emph{competitive ratio}.

\paragraph{Prediction setup.} We consider the setting where there are $\ell$ predictors denoted $P_1,\dots,P_\ell$. At each time $t$, predictor $P_i$ produces a suggestion of a state $\varphi_{it}\in M$ where the algorithm should go. Note that we may think of each $P_i$ itself as an algorithm to serve the request sequence. The case $\ell=1$ of a single predictor was studied in~\cite{AntoniadisCE0S20}.

To evaluate the performance of our algorithms, we consider as benchmark algorithms the best dynamic combination of the predictors. We write $\DYN$ for the cost of the best (offline) algorithm that is in one of the predicted states at each time step:

\begin{align*}
\DYN := \min_{\substack{s_1,\dots,s_T\colon\\ s_t\in \{\varphi_{1t},\dots,\varphi_{\ell t}\}}} \sum_{t=1}^T d(s_{t-1},s_t) + c_t(s_t)
\end{align*}

If $s_t=\varphi_{it}$, we say that the algorithm \emph{follows $P_i$} at time $t$. We define $\DYN^{\le m}$ similarly to $\DYN$, but for an offline algorithm that switches the predictor that it is following at most $m$ times.

For the $k$-server problem, note that $\varphi_{it}$ is a configuration (i.e., a set) of $k$ points. Here, it is natural to consider predictors that are \emph{lazy}, i.e., they move a server only to serve a request; formally, $\varphi_{it}\subseteq \varphi_{i,t-1}\cup\{r_t\}$. In this case, the sequence of predictions produced by $P_i$ can also be encoded by specifying for each time $t$ only the name of the server that should serve the current request. This suggests an alternative definition of a dynamic combination for the $k$-server problem: We write $\widetilde{\DYN}$ for the cost of the best offline algorithm that serves each request $r_t$ using a server named by \emph{any} of the predictors at time $t$.

To clarify the difference between $\DYN$ and $\widetilde{\DYN}$, consider the following example for $k=2$ servers and $\ell=2$ predictors $P_1$ and $P_2$: The servers start in configuration $\{a,b\}$ and the first two requests are to some different points $r_1, r_2\notin\{a,b\}$. Predictor $P_1$ uses the first server for both requests, changing its configuration to $\{r_1,b\}$ and then $\{r_2,b\}$. Predictor $P_2$ uses the second server for both requests, changing its configuration to $\{a,r_1\}$ and then $\{a,r_2\}$. The algorithm achieving cost $\widetilde{\DYN}$ might use the first server for the first request and the second for the second request, thus reaching configuration $\{r_1,r_2\}$ at time $2$, but the algorithm in the definition of $\DYN$ cannot be in that configuration since it is only allowed to be in a configuration where $P_1$ or $P_2$ currently is.%

\paragraph{Full access and bandit access.}
We define two types of learning-augmented algorithms for MTS, depending on the type of access they have to the predictors.
Note that in both cases, they see the input cost function $c_t$.
In the \emph{full access} model, the algorithm receives at each time $t$ as additional input the ordered tuple $(\varphi_{1t},\varphi_{2t},\dots,\varphi_{\ell t})$. In the \emph{bandit-access} model, the algorithm chooses some $i_t\in\{1,\dots,\ell\}$ at time $t$ and only observes the state $\varphi_{i_tt}$ and the (movement + service) cost paid by $P_{i_t}$ at time step $t$.\footnote{In the full access model, the learning-augmented algorithm also knows the cost of each algorithm $P_i$ for all time steps, as it can be deduced from its state at the current and previous time step.} For our algorithms, it does not matter whether $i_t$ is chosen before or after the cost function $c_t$ is observed. In all cases, the learning-augmented algorithm has to choose its own state $s_t$ only \emph{after} observing the full cost function $c_t$ as well as the predicted state(s) for time $t$.

\subsection{Our results}
We begin by stating a negative result concerning the benchmark $\widetilde{\DYN}$ for the $k$-server problem, suggesting that this benchmark is too strong, even if there are only two predictors. %
\begin{theorem}
\label{thm:LB}
For the $k$-server problem on the line metric with full access to two predictors, every deterministic (resp. randomized) learning-augmented algorithm has competitive ratio at least $k$ (resp. $\Omega(\log k)$) against $\widetilde{\DYN}$.
\end{theorem}
Since $k$ is the exact deterministic competitive ratio and $\Omega(\log k)$ is the best known lower bound on the randomized competitive ratio of $k$-server on the line metric \emph{without predictions}~\cite{manasse1990server,chrobak1991dc,Bubeck22kServer}, predictions do not seem useful against this benchmark. We therefore dismiss this benchmark for the remainder.

For the benchmark $\DYN$, we obtain the following result for any MTS (and therefore also the $k$-server problem) by a reduction to the layered graph traversal problem:
\begin{theorem}\label{thm:l^2UB}
For any MTS problem with full access to $\ell$ predictors, there is an $O(\ell^2)$-competitive randomized algorithm against $\DYN$.
\end{theorem}

A similar connection to layered graph traversal (or the equivalent metrical service systems problem) yields the following matching lower bound:
\begin{theorem}\label{thm:l^2LB}
There exist instances of MTS and $k$-server where no randomized learning-augmented algorithm with full access to $\ell$ predictors can achieve a competitive ratio better than $\Omega(\ell^2)$ against $\DYN$.
\end{theorem}

We remark that the algorithm from Theorem~\ref{thm:l^2UB} can be made robust against prediction errors, so that it achieves a cost of at most $O\left(\min\{\rho\cdot\OPT,\ell^2\cdot\DYN\}\right)$, where $\rho$ is the best competitive ratio of the given MTS problem in the setting without predictions. To achieve this, we can take the output of our algorithm and combine it with a $\rho$-competitive algorithm using the methods discussed in~\cite{AntoniadisCE0S20} (or alternatively, by invoking Theorem~\ref{thm:l^2UB} a second time with $\ell=2$, using the output of our algorithm and a $\rho$-competitive algorithm as the two predictors). The exact same combination method can be applied to all of our algorithms to achieve analogous robust versions also of all our subsequent upper bounds, and we will not mention it explicitly each time.


%

As argued above, a more realistic benchmark might be a dynamic combination whose number of switches between predictors is somewhat bounded. In analogy to results about tracking best experts in online learning~\cite{BianchiLugosi}, one might expect a competitive ratio of $1+\epsilon$ against a dynamic combination that switches at most $f(\epsilon)T$ times. %
However, $T$ is not an adequate quantity to express the length of an MTS instance, as dummy tasks of cost $0$ can artificially inflate the length. Instead, we use the value of $\DYN$ to scale our results based on the meaningful length of the instance.

\begin{theorem}
\label{thm:lim_sw}
For any MTS with full access to $\ell$ predictors and any $\epsilon>0$, there is a $(1+\epsilon)^2$-competitive randomized algorithm against $\DYN^{\le m}$ for $m$ as large as $\Omega\left(\frac{\epsilon^2}{\log \ell}\cdot \frac{\DYN}{D}\right)$, where $D$ is the diameter of the underlying metric space.
\end{theorem}

Theorem~\ref{thm:lim_sw} generalizes a result of~\citet{BlumB00}, who showed that a competitive rato of $1+\epsilon$ is achievable against $\DYN^{\le 0}$ (i.e., a static benchmark that does not switch at all). Our approach is based on a connection
to the unfair MTS problem and the bound of Theorem~\ref{thm:lim_sw}
is achieved when using algorithm \OddExponent
by~\citet{BBBT97} as a subroutine.

We remark that it is not necessary for our algorithm to know the value of $\DYN$. Note that the result also holds for large $\epsilon$ (in which case $\epsilon$ would be the dominant term in \mbox{$(1+\epsilon)^2$}). %
The following result shows that the bound obtained is asymptotically optimal for large $\epsilon$, and the dependency on $\DYN$, $D$ and $\ell$ cannot be improved.

\begin{theorem}
    \label{thm:hard_lim_sw}
    For any $\epsilon>0$, there exists an MTS with full access to $\ell$ predictors on which no randomized algorithm can be $(1+\epsilon)^2$-competitive against $\DYN^{\le m}$ if $m\geq \frac{6 (1+\epsilon)^2}{\log \ell} \frac \DYN D$.
\end{theorem}

For the bandit-access model with a limited number of switches%
\footnote{
Earlier versions of this paper contained a simple observation that the result
by \citet{EmekFKR09} allows to transfer the guarantee of Theorem~\ref{thm:l^2UB}
to the bandit-access setting with unlimited number of switching, loosing
an additional factor of $O(\ell)$. This observation is incorrect and a possible variant of
Theorem~\ref{thm:l^2UB} for the bandit-access model requires more sophisticated ideas.
}%
, we obtain the following result.

\begin{theorem}
\label{thm:bandit}
For any MTS with bandit access to $\ell$ predictors and any $\epsilon>0$, there is a $(1+\epsilon)^3$-competitive randomized algorithm against $\DYN^{\le m}$ for $m$ as large as
\[ \Omega\left(\frac{\epsilon^3/\log(2+\epsilon^{-1})}{\ell\log\ell}\cdot\frac{\DYN}{D}\right),\]
 where $D$ is the diameter of the underlying metric space.
\end{theorem}
We use the standard sampling approach (see, e.g., \citep{Slivkins19})
to reduce the bandit setting to the full information setting
and analysis of \citet{BlumB00} for the \share algorithm of
\citet{herbster1998tracking}.

An unexpected consequence of the lower bound in Theorem~\ref{thm:l^2LB}, combined with an upper bound for online covering problems with multiple predictions by~\citet{Anand0KP22}, is a structural insight into possible covering relaxations of the $k$-server problem:
\begin{theorem}\label{thm:covering}
    There exists no configuration-encoding online covering formulation of the $k$-server problem, even in the special case of HST metrics.
\end{theorem}
The terminology in Theorem~\ref{thm:covering} is defined formally in Section~\ref{sec:covering}. Intuitively, by ``configuration-encoding'' we mean that the server configuration at time $t$ is uniquely determined by the values of variables involved in constraints for time $t$. Since configuration-encoding online covering formulations exist for the weighted paging problem, Theorem~\ref{thm:covering} yields a structural separation between star and HST metrics for the $k$-server problem. We are not aware of any similar structural impossibility results about LP formulations of online problems.

\subsection{Organization}

The remainder of the paper is organized as follows. In Section~\ref{sec:unbounded}, we study upper and lower bounds on the achievable competitive ratio against $\DYN$ and $\widetilde{\DYN}$ with full access to predictors. In Sections~\ref{sec:limited} and~\ref{sec:hard_limit}, we respectively show positive and negative results when using $\DYN^{\le m}$ as a benchmark. We focus on the bandit-access
setting in Section~\ref{sec:bandit}. Finally, in Section~\ref{sec:covering} we prove the impossibility result regarding covering formulations for the $k$-server problem.

\section{Unbounded number of switches}
\label{sec:unbounded}
The goal of this section is to show tight bounds against $\DYN$ as well as a lower bound for $k$-server against $\widetilde{\DYN}$ in the full access model.

\subsection{Tight bounds against \texorpdfstring{$\DYN$}{DYN}}
We start by showing that in the full access model, the competitive ratio against $\DYN$ (with unlimited number of switches) is $O(\ell^2)$ and $\Omega(\ell^2)$, proving Theorems~\ref{thm:l^2UB} and~\ref{thm:l^2LB}.

The problem of combining predictors $P_1, \dotsc, P_\ell$ on an MTS  instance ($\ell$-MTS for short) can be formulated as a classical MTS on the same underlying metric space: it is enough to modify the losses:
$\ell'_t(s) := \ell_t(s)$ if $s$ is a state of some predictor
and $\ell'_t(s) := +\infty$ otherwise.
This formulation does not yet make the problem easier: the underlying metric space remains the same with the same number of points $n$, seemingly keeping the complexity of the problem the same as solving the input instance directly.
However, having finite loss only on at most $\ell$ states at a time
allow us to reduce this problem to a more structured variant of MTS called $\ell$-width \emph{layered graph traversal (LGT)}. We will show this in a similar way to a reduction from metrical service systems to LGT from~\cite{FiatFKRRV98}.

In layered graph traversal (LGT) we are given a graph with non-negative edge weights, and a searcher that starts at a designated vertex $s$. The graph has the property that its vertices can be partitioned into layers $L_0:=\{s\}, L_1, L_2\dots $ such that any edge connects vertices of two consecutive layers. In $\ell$-width LGT, one has $|L_t|\le \ell$ for all $t$. The problem is online, meaning that the searcher is only aware of the edges (and corresponding weights) adjacent to the layers visited so far. Each traversal of an edge by the searcher incurs a cost equal to the weight of that edge. The goal is to move the searcher along the edges to a target vertex in the last layer. The cost is the distance travelled by the searcher.

It was shown recently that $\ell$-width LGT admits an $O(\ell^2)$-competitive randomized algorithm~\cite{BubeckCR22}.

\begin{proof}[Proof of Theorem~\ref{thm:l^2UB}]
 Consider an instance $I$ of $\ell$-MTS. We can construct a corresponding instance $I'$ of $\ell$-width layered graph traversal as follows. Every layer $L_t$, $t\ge 1$ in $I'$ consists of exactly $\ell$ vertices $v_{1t},v_{2t},\dots v_{\ell t}$ where intuitively vertex $v_{it}$ corresponds to the state $\varphi_{it}$ of predictor $P_i$ at time $t$. The edges between any two consecutive layers form a complete bipartite graph, where the weight of edge $(v_{i,t-1},v_{jt})$ is set to $d(\varphi_{i,t-1},\varphi_{jt}) + c_{t}(\varphi_{jt})$. Finally, all vertices of the last layer $L_T$ constructed in this way are connected to single target vertex in layer $L_{T+1}$ with edges of weight $0$. 

It can be easily verified that $I'$ is a feasible $\ell$-width layered graph traversal instance. Furthermore any solution to $I'$ can be naturally (and in an online-fashion) transformed into a corresponding solution for $I$: If the searcher in $I'$ moves to vertex $v_{it}$ when layer $L_{t}$ is revealed, then the corresponding $t$'th request in $I$ is served in state $\varphi_{it}$. Note that by construction the costs of the two solutions are exactly the same, as going back to a previous layer is never beneficial because $d$ is a metric. Similarly one can apply the opposite transformation to the offline solution achieving cost $\DYN$ for instance $I$ to obtain an offline solution of the same cost for instance $I'$ for $\ell$-width LGT.

The result follows, by the $O(\ell^2)$-competitive algorithm for $\ell$-width LGT by~\citet{BubeckCR22}.
\end{proof}

On the other hand, our lower bound in Theorem~\ref{thm:l^2LB} can be derived via the \emph{Metrical Server Systems (MSS)} problem~\citep{ChrobakL91}, which is in fact equivalent to LGT. In Metrical Server Systems (MSS)~\citep{ChrobakL91}, a server can move between the points of a metric space. In each round it is presented with a request, which consists of $w$ points of the metric. In response, the server has to move to one of these $w$ points. The goal is to minimize the total distance traversed by the server. 

The proof of Theorem~\ref{thm:l^2LB} follows by observing the relationship between MSS and respectively $\ell$-MTS and $k$-server.

\begin{proof}[Proof of Theorem~\ref{thm:l^2LB}]
    By the result of~\citet{Bubeck22kServer} that any (randomized) algorithm  for MSS with $w=\ell$ is  $\Omega(\ell^2)$-competitive, and since $k$-server is an MTS, it suffices to show that MSS with $w=\ell$ can be reduced to the $k$-server problem with full access to $\ell$ predictors.

     Consider an arbitrary input instance $I$ to MSS on an $n$-point metric space $\mathcal{M}$, where the server initially is at a point $p_0\in\mathcal{M}$. We construct a $k$-server instance $I'$ on the same metric space with $k=n-1$, so that at each time there is exactly one point not covered by a server (called the \emph{hole}). The initial hole is $p_0$. Fix a learning-augmented algorithm $A'$ for $k$-server with full access to $\ell$ predictors. We define an algorithm $A$ for MSS via $A'$ as follows. Whenever a request to a set $W_t$ arrives in $I$ in round $t$, in $I'$ we repeatedly issue many requests at all the points outside $W_t$, so that any competitive algorithm is forced to move its hole eventually to a point in $W_t$. For each of these requests, let the $\ell$ predictors collectively have their holes at each of the points in $W_t$. After sufficiently many such requests, the hole of $A'$ must move to a point $q\in W_t$ with probability arbitrarily close to $1$. $A$ serves the original request $W_t$ in $I$ by moving to this point $q$.
     
     By the triangle inequality, the cost of $A$ is at most the cost of $A'$. On the other hand, the optimal offline cost for instance $I$ is equal to the cost of $\DYN$ on instance $I'$. Thus, if $A'$ were $o(\ell^2)$-competitive against $\DYN$, then $A$ would be $o(\ell^2)$-competitive for MSS, contradicting the result of~\citet{Bubeck22kServer}.
    \end{proof}

\subsection{Lower bound for \texorpdfstring{$k$}{k}-server against \texorpdfstring{$\widetilde{\DYN}$}{\~DYN}}
\label{app:tilde}

In their recent work, \citet{Anand0KP22} expressed the belief that their framework for multiple predictions can be applied to problems other than set-cover, (weighted) caching and facility location. In particular ``it would be interesting to consider the $k$-server problem with multiple suggestions in each step specifying the server that should serve the new request''. For the benchmark~$\DYN$, we gave a tight answer of $\Theta(\ell^2)$ in Section~\ref{sec:unbounded}. But also for the benchmark~$\widetilde{\DYN}$, we show that there exist instances on which such predictors are not beneficial.

\begin{proof}[Proof of Theorem~\ref{thm:LB}]
    Consider the line metric with $k+1$ distinct points indexed from left to right as $p_1,p_2,\dots p_{k+1}$.  As mentioned in the introduction, we can restrict to \emph{lazy} algorithms. Furthermore, we can assume without loss of generality that for any algorithm, two servers never reside at the same point simultaneously.
    
    It is known~\cite{manasse1990server} (resp.~\cite{Bubeck22kServer}) that any deterministic (resp. randomized) online algorithm has competitive ratio at least $k$ (resp. $\Omega(\log k)$) on any metric space of at least $k+1$ points. Let the set of servers be indexed $s_1,\dots s_k$ from left to right in their initial configuration. In order to prove the theorem, it is sufficient to show that there exists an optimal solution $\OPT$ on which every request to a point $p_i$ is served by servers $s_{i-1}$ or $s_i$.
    The result then follows by having the two predictors produce suggestions $s_{i-1}$ and $s_i$ respectively (for the border cases when $i=1$ or $i=k+1$ we have both predictors suggest $s_1$ or $s_k$ respectively) whenever point $p_i$ is requested.  This implies that $\OPT$ is an algorithm that serves each request $p_i$ using a server named by one of the predictors in that round and thus by definition cannot have cost lower than $\widetilde{\DYN}$.

    For the sake of contradiction assume that the claim is wrong, that is, there exists some optimal algorithm $\OPT$ which serves some request $p_i$ in round $r$ with a server $s_j$ such that $j<i-1$ or $j>i$. In case there are more such algorithms, let $\OPT$ be one maximizing $r$. We assume $j<i-1$ as the other case is symmetrical. We modify $\OPT$ to obtain an algorithm $\OPT'$ as follows. The rounds up to (excluding) $r$ are served identically to $\OPT$. The request to $p_i$ in round $r$ is served by server $s_{i-1}$ (which currently resides at $p_{i-1}$). 
     At the same time, the server $s_{j}$ moves to $p_{i-1}$, so that servers $s_{i-1}$ and $s_j$ are swapped compared to the current state of $\OPT$. In later rounds, $\OPT'$ imitates $\OPT$ but exchanging the roles of servers $s_{i-1}$ and $s_{j}$. This gives an algorithm  with the same cost as $\OPT$ thus contradicting the definition of $r$.%

\end{proof}

If the learning-augmented algorithm is forced to follow a predictor's suggestion in each step, then the above proof extends to arbitrary metric spaces by fixing $k+1$ points $p_1,\dots,p_{k+1}$ to be used for the lower bound instance, and using two predictors that keep their $i$th server always at one of the two points $p_i$ and $p_{i+1}$ so that the set of usable edges constitutes a path.

\section{Limited number of switches}

\label{sec:limited}

\newcommand{\combineA}{\ensuremath{\textsc{Combine}_{\bar A}}\xspace}

Consider predictors $P_1, \dotsc, P_\ell$ for some MTS instance $I$ with diameter $D$.
In order to construct an algorithm for combining these predictors,
we create a new MTS instance $U$ on a uniform metric space with $\ell$ points,
each corresponding to one of the predictors.
At each time step, we calculate, for each $i=1, \dotsc, \ell$, the cost $f_t(P_i)$ incurred by predictor
$P_i$ on instance $I$ at time $t$,
which includes both the movement and service cost of $P_i$,
and issue the cost function $c_t^U$ such that
$c_t^U(i) = \frac1D f_t(P_i)$.

A solution to the instance $U$ produced by some algorithm $\bar A$
can be translated to a combination of the
predictors: whenever $\bar A$ resides at state $i$,
we move to the current state of the predictor $P_i$.
Service costs in $U$ correspond to the scaled costs of the individual
predictors. Therefore, if $\bar A$ always resides in state $i$,
its total cost will be $\frac1D$ times the total cost of $P_i$
serving the instance $I$.
However, this translation does not preserve switching costs:
neither for our algorithm nor for the optimal combination.
While moving from $i$ to $j$ in instance $U$ always costs $1$,
switching from $P_i$ to $P_j$ may cost anything between $0$ and $D$.
Algorithm~\ref{alg:combination} summarizes this translation.

\begin{algorithm2e}
\label{alg:combination}
\caption{$\combineA(P_1, \dotsc, P_\ell)$}
$D :=$ diameter of the space\;
\ForEach{$t=1, \dotsc, T$}{
	$c_t^U(i) = D^{-1}f_t(P_i)$ for each $i=1, \dotsc, \ell$\;
	$i_t :=$ state of $\bar A$ after seeing
		$c_1^U, \dotsc, c_t^U$\;
	Move to the state of $P_{i_t}$\;
}
\end{algorithm2e}

We choose algorithm $\bar A$ based on the following performance metric.

\begin{definition}[Unfair competitive ratio]
\label{def:unfair}
Let $r>0$.
We say that an MTS algorithm $A$ is $r$-unfair competitive if there
is a constant $\alpha \geq 0$ such that for any instance the cost incurred by $A$ is
\begin{align*}
\cost(A)
	&= \sum_{t=1}^T (c_t(x_t) + d(x_{t-1},x_t))\\
	&\leq R \cdot \min_{y\colon y_0=x_0}\{ \sum_{t=1}^T (c_t(y_t) + r d(y_{t-1},y_t))\}
	+ \alpha,
\end{align*}
where $x$ is the solution produced by the algorithm and
the minimum in the right-hand side is the cost of the optimal solution
whose movement costs are scaled by factor $r$.
We call $R$ the $r$-unfair competitive ratio of $A$.
\end{definition}

Unfair ratios are usually considered with $r \leq 1$, i.e., the reference optimal solution pays
cheaper costs for its movement than the algorithm,
since this setting is important in the design of general
algorithms for MTS \cite{FiatM03,BubeckCLL21}.
In our case, we are trying to prevent the optimum
solution from moving too much, that's the intuition
due to which we are interested in $r>1$.
The algorithm \OddExponent by \citet{BBBT97} achieves the following
bound also with $r\geq 1$. We denote by $r(\epsilon)$ the minimal $r$ such that
there is an algorithm with $r$-unfair competitive ratio $1+\epsilon$ for the $\ell$-point uniform metric.
A description of \OddExponent can be found in Appendix~\ref{app:oddexp}.

\begin{proposition}[\citet{BBBT97}]
\label{prop:odd_exp}
Given $r$, there is an algorithm for the $\ell$-point uniform metric space
with $r$-unfair competitive ratio $1+\frac1r 2e\ln\ell$.
This gives $r(\epsilon) = O(\epsilon^{-1} \ln\ell)$.
\end{proposition}

The following lemma relates the cost of $\combineA(P_1, \dotsc, P_\ell)$
to the cost of an optimal combination which has to pay a fixed large cost
for every switch between two predictors.

\newcommand{\raeps}{r_{\bar A}(\epsilon)\xspace}
\begin{lemma}
\label{lem:augment_sw}
Let $\bar A$ be an algorithm for uniform MTS and $\raeps$ be such that
the $\raeps$-unfair competitive ratio of $\bar A$ is $(1+\epsilon)$,
for some $\epsilon>0$.
Let $\DYN_{\rho}$ denote the optimal cost of a combination of predictors
$P_1, \dotsc, P_\ell$ which pays $\rho= 2D\raeps$ for each switch between two predictors.
Then $\combineA(P_1, \dotsc, P_\ell)$ is $(1+\epsilon)$-competitive
with respect to $\DYN_\rho$.
\end{lemma}
\begin{proof}
Denote $\OPT^U_{\raeps}$ the cost of the optimum solution for $U$
which pays $\raeps$ instead of 1 for each movement.
We know that the cost of $\bar A$ is at most $(1+\epsilon)\OPT^U_{\raeps}
+\alpha$ for some constant $\alpha$.

Now, $\combineA$ pays $f_t(P_i)$ when following $P_i$ or,
if there was a switch,
at most $D + f_t(P_i)$.
In the same situation, $\bar A$ pays
$D^{-1} f_t(P_i)$ and $1+ D^{-1}f_t(P_i)$ respectively.
Therefore, the total cost of $\combineA$ is at most
\begin{align*}
\sum_{t=1}^T D \cost_t(\bar A)
&\leq D \cdot \big( (1+\epsilon) \OPT^U_{\raeps} + \alpha\big).
\end{align*}
To show that $\OPT^U_{\raeps} \leq \frac1D \DYN_{\rho}$,
we translate $\DYN_\rho$ into a (possibly suboptimal) solution
on instance $U$ as follows:
If $\DYN_\rho$ follows $P_i$ at time $t$ and pays cost
$f_t(P_i)$, we stay at state $i$ in $U$ and pay cost
$\frac1D f_t(P_i)$.
Otherwise, $\DYN_\rho$ switches from $P_j$ to $P_i$
at time $t$ and pays cost
$2D\raeps + srv(P_i)$, where $srv(P_i)$ denotes the service cost paid by $P_i$. We move from state $j$ to $i$ in $U$
and pay
$r(\epsilon) + \frac1D f_t(P_i) \leq r(\epsilon) +1 + \frac1D srv(P_i)$, because the moving cost of $P_i$ is at most $D$.
In both cases, the cost incurred by $\DYN_\rho$ was $D$ times
larger than the constructed solution on $U$,
implying $\OPT_{\raeps}^U \leq \frac1D \DYN_{\rho}$.
\end{proof}

Theorem~\ref{thm:lim_sw} follows from the following lemma translating
the competitive ratio with respect to $\DYN_\rho$ to a competitive
ratio with respect to $\DYN^{\le m}$.

\begin{lemma}
\label{lem:aug_limit}
Let $\epsilon>0$ and $\rho>0$.
If an algorithm $A$ is $(1+\epsilon)$-competitive against
$\DYN_\rho$, then it is $(1+\epsilon)^2$-competitive against
$\DYN^{\le m}$ for any $m\leq \epsilon\DYN/\rho$.
\end{lemma}
\begin{proof}
Let us denote $\alpha$ such that $\cost(A)\leq (1+\epsilon)\DYN_\rho + \alpha$.
Relating its cost to $\DYN^{\le m}$
for any $m\leq \epsilon \DYN/\rho$, we have
\begin{align*}
\cost(A) &\leq (1+\epsilon) \DYN_\rho + \alpha\\
&\leq (1+\epsilon) (\DYN^{\le m} + \epsilon \DYN) + \alpha\\
&\leq (1+\epsilon)^2 \DYN^{\le m} + \alpha,
\end{align*}
because
$\DYN_\rho \leq \DYN^{\le m} + m\rho$ and
$\DYN\leq \DYN^{\le m}$.
\end{proof}

Using Lemma~\ref{lem:augment_sw} and choosing $\bar A$ from
Proposition~\ref{prop:odd_exp}, we get that $\combineA$ is
$(1+\epsilon)^2$-competitive with respect to $\DYN^{\le m}$ whenever
$m\leq \frac{\epsilon^2}{4De\ln \ell} \DYN$, as claimed by
Theorem~\ref{thm:lim_sw}.

\section{Hardness for limited number of switches}
\label{sec:hard_limit}

In this section we show Theorem~\ref{thm:hard_lim_sw}, stating that the bound on the maximum number of allowed switches $m$ given by Theorem~\ref{thm:lim_sw} is tight up to a constant factor, for fixed $\epsilon$. In particular, the asymptotic dependence on $\ell$, $D$, and $\DYN$ is optimal.

The randomized construction we use in this section is inspired by the classical coupon collector lower bound for MTS~\cite{BorodinLS92}. %
We consider a uniform\footnote{I.e., the distance between any two different points is $1$.} metric space with $\ell$ points. There are also $\ell$ predictors, the $i$-th of them predicting to always stay at point $i$. Let $\sigma_1, \ldots, \sigma_T$ be $T$ independent random variables, each drawn uniformly from the metric space. Let $\alpha \in(0, 1]$ be a parameter. At time step $t$, the cost function is
\[ c_t(x) = \begin{cases} 1 & \text{if } x = \sigma_t, \\ \alpha / \ell & \text{if } x \neq \sigma_t. \end{cases} \]

In each step, any online algorithm (even given access to the above predictors, whose predictions are independent from the random instance) has expected cost of at least $1/\ell$, since with probability $1/\ell$ the random point $\sigma_t$ falls on the old state of the algorithm, and the algorithm either moves and incurs moving cost $1$ or stays and incurs service cost $1$. After $T$ steps, the expected total cost of an algorithm $A$ is at least $\E[\cost(A)]\ge T / \ell$.

\newcommand{\const}{\ensuremath{\mathit{const}}\xspace}

Clearly, $\DYN \geq T \alpha / \ell$. Let $m = \frac{2 \DYN}{\alpha \ln \ell} \geq \frac{4T}{\ell \ln \ell}$. We will upper bound the expected value of $\DYN^{\le m}$ by considering the following offline strategy. Whenever $\sigma_t$ hits the currently followed predictor, switch to the predictor that will be hit furthest in the future (i.e., akin to Belady's rule for the caching problem), unless the switching budget $m$ has already run out.

Let $X$ be the random variable denoting the number of steps from a given switch until the next switch. By a coupon-collector analysis,
\[ \E[X] = \sum_{i=1}^{\ell-1} \E[\Geo(\nicefrac{i}{\ell})] = \sum_{i=1}^{\ell-1} \nicefrac{\ell}{i} > \ell\ln\ell,\]
where $\Geo(p)$ denotes a geometrically distributed random variable with success probability $p$. Moreover,
\[\Var(X) = \sum_{i=1}^{\ell-1} \Var(\Geo(\nicefrac{i}{\ell})) = \sum_{i=1}^{\ell-1} \frac{1 - \frac{i}{\ell}}{\big(\frac{i}{\ell}\big)^2} \leq \frac{\pi^2}{6} \cdot \ell^2.\]
Let $Y$ be the random variable denoting the expected number of switches until time $T$ when ignoring the upper bound $m$. The central limit theorem for renewal processes shows that
\begin{align*}
\lim_{T\to\infty}\frac{\E[Y]}{T}&=\frac{1}{\E[X]}<1/(\ell\ln\ell)\quad\text{ and }\\
\lim_{T\to\infty}\frac{\Var[Y]}{T}&=\frac{\Var[X]}{\E[X]^3}<1.
\end{align*}
Therefore, for large enough $T$,
\begin{align*}
\E[Y] &\leq \frac{T}{\ell \ln\ell} \quad \text{and} \quad \Var[Y] \leq T.
\end{align*}
For large enough $T$, the switching budget $m \geq \frac{2T}{\ell \ln \ell}$ is at least $\E[Y] + \nicefrac{\sqrt{T}}{\ell\ln\ell} \cdot \sqrt{\Var(Y)}$. Hence, by Chebyshev's inequality, the probability of running out of the switching budget can be upper bounded by $P(Y > m) \leq  \ell^2 \ln^2\ell / T$, and in that case the expected total cost of following a fixed predictor can be upper bounded by $T (1 + \alpha) / \ell$. In the event the strategy does not run out of the switching budget, the total service cost is $T \alpha / \ell$ and the expected movement cost is at most $T / (\ell \ln\ell)$. Summing up,
\[\E[\DYN^{\le m}] \leq T \alpha / \ell + T / (\ell \ln\ell) +  P(Y > m) T (1 + \alpha) / \ell.\]
For $\ell$ and $T$ large enough, we get $\E[\DYN^{\le m}] < 3\alpha T/\ell$.

Since for any online algorithm $A$ (with predictions) we have $\E[\cost(A)]\ge T / \ell$, we conclude that for any constant $c$ there exists $T$ large enough such that $\E[\cost(A)]\ge \E[DYN^{\le m}]/(3\alpha) + c$ for the random request sequence (and hence there also exists a deterministic sequence for which the inequality holds). We conclude that no algorithm can be better than $(\nicefrac{1}{3\alpha})$-competitive against a combination of predictors that allows $\frac{2}{\alpha \ln \ell} \DYN$ switches, on a metric space with diameter $D=1$.
The generalization to arbitrary values of $D$ can be made by scaling distance and service costs by a factor $D$ and replacing $\DYN$ by $\nicefrac{\DYN}{D}$ in the definition on $m$. Setting $\alpha = 1 / (3\cdot (1 + \epsilon)^2)$ yields $m=\frac{6 (1+\epsilon)^2}{\ln \ell} \frac\DYN D$, proving Theorem~\ref{thm:hard_lim_sw}.

\section{Bandit access to predictors}
\label{sec:bandit}

\newcommand{\loshatt}{{\ensuremath{\hat\ell_t}}\xspace}
\newcommand{\nalg}{\ensuremath{\ell}} %
\newcommand{\banditcombine}{\textsc{BanditCombine}\xspace}
\newcommand{\banditcombinep}{\ensuremath{\textsc{BanditCombine}'}\xspace}

In this section,
we focus on a more restrictive setting inspired by the
multi-armed bandit model, and motivated by the fact that querying many predictors may be expensive: at each time $t$, the algorithm still has access to
the full cost function $c_t$ of the original MTS instance,
but it is able to query the state of only one predictor.
Only after selecting which predictor $j$ to query at time $t$, the
algorithm is aware of its state $\varphi_{jt}$ and of its (movement + service) cost
$f_t(j)$ incurred at this time step.
Then, the algorithm chooses its own state, which does not necessarily have to be
$\varphi_{jt}$.

%

\subsection{Limited number of switches}

We consider an MTS instance of finite diameter $D$.
We assume that, at each time step $t$,
there is a state $x$ such that $c_t(x) = 0$.
This is without loss of generality: we can modify the cost function
by subtracting $\min_x\{c_t(x)\}$ from the cost of each state
at time $t$. Since this discounts the cost of all algorithms (including the benchmark) by the same additive quantity, the competitive ratio on the original instance is no larger than on the modified instance.
We can further assume that $f_t(i)\leq 2D$ for each $i$ and $t$:
if this is not the case,
we move to the state with cost 0 (guaranteed by the assumption above),
serve the task there and move back to $\varphi_{it}$, paying at most $2D$ in total.

Let $\bar A$ be an algorithm for unfair MTS on uniform metric spaces.
Algorithm~\ref{alg:banditcombine}
for the bandit-access model creates a suitable MTS instance on the uniform metric
space of size $\ell$
and uses $\bar A$ to choose which predictor $a_t$ to follow,
moving to state $b_t = \varphi_{a_tt}$.
However, with a small probability $\gamma$, it does not query the state of
$a_t$, querying a random predictor instead -- we call this an exploration step.
This is a common technique in multi-armed bandits, see~\cite{Slivkins19} for instance.
During an exploration step at time $t$, it makes
greedy steps from $b_{t-1}$ to a state $g_t$. Once serving the
cost function at $g_t$, it returns back to $b_t=b_{t-1}$.
The algorithm is described in Algorithm~\ref{alg:banditcombine}.
It requires a parameter $0 < \gamma < 1/4$ which denotes the exploration rate.

\begin{algorithm2e}
\caption{\banditcombine($P_1, \dotsc, P_\ell$)}
\label{alg:banditcombine}
Select $X\subseteq [T]$ by choosing each $t\in[T]$ independently with probability $\gamma$\;
For each $t\in X$: choose $i_t\in \{1, \dotsc, \ell\}$ uniformly at random\;
\For{$t=1, \dotsc, T$}{
	\uIf(\tcc*[f]{exploration step}){$t\in X$}{
		Query predictor $i_t$\;
		set $\hat f_t(i_t) := f_t(i_t)/(2D)$ and $\hat f_t(j)=0\, \forall j\neq i_t$\;
		feed $\hat f_t$ into $\bar A$\;
		serve the request at $g_t:= \min_x\{d(b_{t-1},x)+c_t(x)\}$
		\nllabel{alg:banditcombine_greedy}
		\tcc*{greedy step}
		return to $b_t:=b_{t-1}$\;
		\nllabel{alg:banditcombine_greedy2}
	}
	\Else(\tcc*[f]{exploitation step}){
		Feed $\hat f_t := 0$ into $\bar A$\;
		\nllabel{alg:banditcombine_exploit}
		Query predictor $a_t$ chosen by $\bar A$ and set $b_t:=\varphi_{a_tt}$\;
	}
}
\end{algorithm2e}

\begin{observation}
\label{obs:unbiased}
Let $X_t = X\cap [t]$ and $I_t = (i_t)_{t\in X_t}$.
Since each $t$ was added to $X$ independently at random and
$i_t$ was also chosen independently, we have
\[ \E[\hat f_t | X_{t-1}, I_{t-1}] =
\E[\hat f_t] = \frac{\gamma}{2D\ell} f_t.\]
\end{observation}

We choose $\bar A$ to be the algorithm \share
by \citet{herbster1998tracking}, which has the following advantages over
\OddExponent.
First, it does not require splitting cost functions as far as they are bounded by $1$. Splitting the cost functions is problematic
in the bandit-access model, since we are allowed to query only one algorithm per
time step. Second, it chooses its state without lookahead, i.e.,
its state at time $t$ depends only on cost functions $c_1, \dotsc, c_{t-1}$.
See Appendix~\ref{app:algs} for a description of both algorithms.

\begin{proposition}[\citet{BlumB00}]
\label{prop:share}
Given $r>0$, configure \share with $\alpha = 1/(2r+1)$ and $\beta = \max\{1/2, 1-\gamma\}$, where $\gamma = \frac1r \ln (\ell/\alpha)$.
Then, in the uniform metric space on $\ell$ points, \share has
$r$-unfair competitive ratio at most
\[ R_\ell^r := 1 + \frac8r (\ln \ell + \ln (2r+1)).\]
\end{proposition}

For $\epsilon>0$, let $r(\epsilon) = O(\epsilon^{-1}\ln(2+\epsilon^{-1})\ln \ell)$ be such that $R_{\ell}^{r(\epsilon)}=1+\epsilon$.

First, we analyze the following variant \banditcombinep which
queries two predictors during exploration steps.
I.e., instead of the greedy step (Lines~\ref{alg:banditcombine_greedy},
\ref{alg:banditcombine_greedy2}),
it makes an additional query to
the predictor $a_t$ suggested by $\bar A$ and moves to
$b'_t:=\varphi_{a_tt}$.

\begin{lemma}
\label{lem:bandit_prime}
Choose $\bar A$ to be an algorithm for MTS on the $\ell$-point uniform metric
whose $\rho$-unfair competitive ratio is $R$ with additive term $\alpha$ for instances with bounded cost functions $c_t\le 1$
and which does not use lookahead, i.e., its state $a_t$ depends only
on costs up to time $t-1$.
Then the expected cost of \banditcombinep is at most
\[ R\cdot \DYN_{\rho'} + \frac{2D\ell}{\gamma} \alpha, \]
where $\rho' = \frac{3D\ell\rho}{\gamma}$.
\end{lemma}
\begin{proof}
Let $p_1, \dotsc p_T \in [0,1]^\ell$ be the probability distributions over the state of
$\bar A$ at time steps $1, \dotsc, T$.
Note that the expected service cost paid by \banditcombinep at time $t$ equals $\langle f_t,p_t\rangle$, the scalar product between the vectors representing the predictor costs and the predictor probabilities. We define $d_E$ as the earth mover's distance between two probability vectors over the uniform metric, which represents the total probability mass that has to be shifted to transform one vector into the other. We may assume that the probability that $\bar A$ changes states between times $t-1$ and $t$ equals $d_E(p_{t-1},p_t)$, as this is the best way to match given probability vectors, which are the core of the algorithm $\bar A$.
We have

\begin{align*}
\E\left[ \sum_t \langle\hat{f}_t,p_t\rangle \right]
&= \sum_t \E_{X_{t-1},I_{t-1}}\left[ \langle\E[\hat{f}_t|X_{t-1},I_{t-1}],p_t\rangle\right]\\
&= \frac{\gamma}{2D\ell}\E\left[\sum_t \langle f_t, p_t\rangle\right]\\
\end{align*}

The first equation separates the random events before and after time $t$, as
$p_t$ and $p_{t-1}$ depend solely on $X_{t-1}$ and $I_{t-1}$ while $\hat f_t$
is uncorrelated to these events. This allows to use Observation~\ref{obs:unbiased} to express $\hat f_t$ in function of $f_t$.
The last expectation is an upper bound on the expected cost paid by \banditcombinep, excluding the cost it pays for switching between predictors.%

Now, for any instantiation of $\hat f$ and for any solution $q$,
we have
\begin{align*} 
\sum_t (\langle\hat{f}_t,p_t\rangle + d_E(p_{t-1},p_t)) \leq
R \sum_t (\langle\hat{f}_t,q_t\rangle + \rho\, d_E(q_{t-1},q_t)) + \alpha
\end{align*}
by the performance guarantees of $\bar A$.
Therefore, the total cost of $\banditcombinep$ is at most (noting that $\gamma<2\ell$)
\begin{align*}
\E\left[\sum_t (\langle f_t,p_t\rangle + D\,d_E(p_{t-1},p_t)\right]
&\leq \frac{2D\ell}{\gamma}\E\left[ \sum_t (\langle\hat{f}_t,p_t\rangle + d_E(p_{t-1},p_t))\right]\\
	&\leq \frac{2D\ell}{\gamma} \E\left[R \sum_t (\langle\hat{f}_t,q_t\rangle + \rho\, d_E(q_{t-1},q_t)) + \alpha\right]\\
	&= R \sum_t (\langle f_t,q_t\rangle + \frac{2D\ell\rho}{\gamma}\, d_E(q_{t-1},q_t)) + \frac{2D\ell}{\gamma}\alpha.
\end{align*}
The sum in the right-hand side is a lower bound on the cost of the solution
$q$
because, when switching to predictor $j$ at time $t$,
it has to pay $\rho'$ but may not need to pay
the moving cost (at most $D$) of the predictor $j$ which is included in
$f_t(j)$. I.e., its expected cost at time $t$ is at least
$\langle f_t, q_t\rangle + (\rho' - D)d_E(q_{t-1}, q_t)$.
Since
the equation above is true for any $q$, and $\DYN_{\rho'}$ corresponds
to the cost of the best solution $q$, we get the desired bound.
%
%
\end{proof}

\begin{lemma}
\label{lem:bandit}
The cost of $\banditcombine$ is at most
\[(1+6\gamma)\cost(\banditcombinep).\]
\end{lemma}

\begin{proof}
Note that states of \banditcombine and \banditcombinep are the same
during exploitation steps, i.e., $b_t = b'_t$ for all exploitation steps.
At time step $t$, the cost paid by \banditcombinep is
$C'_t = d(b'_{t-1}, b_t') + c_t(b_t')$
and its total cost is $\sum_{t=1}^T C'_t$.

To bound the cost paid by \banditcombine we define
$C_t := d(b'_{t-1}, b_t) + c_t(b_t)$ if $t$ is an exploitation
step, note that $b_t=b'_t$ in such case.
For exploration steps, we define
\[ C_t := d(b'_{t-1},g_t) + c_t(g_t) + d(g_t, b_t) + d(b_t,b'_t). \]
The last term is to simplify the analysis:
at step $t+1$, \banditcombine moves by a distance
$d(b_t, b_{t+1}) \leq d(b_t, b'_t) + d(b'_t, b_{t+1})$ and we split
this cost counting the first part to $C_t$ and the second one
to $C_{t+1}$. The total cost of \banditcombine
is then at most $\sum_{t=1}^T C_t$ and we have $C_t = C'_t$ for every
exploitation step $t$.

Consider an exploration step $t$ which is the $i$th consecutive exploration step,
i.e. step $a=t-i$ is an exploitation step (or $a=0$) and all steps from
$a+1$ until $t$ are exploration.
Observe that \banditcombine does not change $b_t$ during exploration steps
and we have $b_t = b_a = b'_a$.
We can bound $C_t$ as follows. We have
\begin{align*}
d(b_t, b'_t) = d(b_a,b'_t) &\leq C'_{t-i+1} + \dotsb + C'_t\\
d(b'_{t-1}, g_t) &\leq d(b'_{t-1}, b_t) + d(b_t, g_t)\\
d(b_t,g_t) + c_t(g_t) &\leq d(b_t,b'_t) + c_t(b'_t)
\end{align*}
The first inequality holds because \banditcombinep needs to move from $b_t=b_{t-i}$
to $b'_t$. The second one follows from the triangle inequality and the last one
comes from the definition of $g_t$.
In total, we have
\[ C_t \leq \big[d(b'_{t-1}, b_t)\big] + \big[d(b_t, g_t) + c_t(g_t) + d(g_t, b_t)\big]
	+ \big[ d(b_t,b'_t)\big].
\]
The first and the third brackets are bounded by $\sum_{j=a+1}^t C'_j$,
because $b_t=b'_a$ and \banditcombinep moves from $b_t$ to $b'_{t-1}$ and
$b'_{t}$ during that time.
By the choice of $g_t$, the second bracket is at most 
$2(d(b_t,b'_t) + c_t(b'_t)) \leq 2\sum_{j=a+1}^t C'_j$.

Since the probability of $t$ being the $i$th exploration step in a row is
at most $(1-\gamma)\,\gamma^i$, the expected cost of \banditcombine
is at most
\begin{align*}
(1-\gamma) \sum_t C'_t +
	\sum_{i=1}^{T} (1-\gamma)\gamma^i \cdot \sum_t 4 \sum_{j=t-i}^t C'_j\\
\leq (1-\gamma) \sum_t C'_t + \sum_{i=1}^{T} \gamma^i 4i \sum_t C'_t 
\end{align*}
which is at most $(1+6\gamma)\sum_t C'_t$ for $\gamma\leq 1/4$.
\end{proof}

\begin{proof}[Proof of Theorem~\ref{thm:bandit}]
We choose $\gamma=\frac{\min\{1,\epsilon\}}{6}$. By Lemma~\ref{lem:bandit_prime}, \banditcombinep is $(1+\epsilon)$-competitive against $\DYN_{\rho'}$ for $\rho'= \frac{3D\ell}{\gamma} \rho = \frac{3D\ell}{\gamma} r(\epsilon)$. Hence, by Lemma~\ref{lem:aug_limit} it is $(1+\epsilon)^2$-competitive against $\DYN^{\le m}$ for any
$m \leq \frac{\epsilon \gamma}{3\ell r(\epsilon)} \frac{\DYN}{D}$.
By Proposition~\ref{prop:share}, the latter quantity is
$\Omega\left(\frac{\epsilon^{3}}{\ell\ln \ell\ln(2+\frac1\epsilon)}\cdot\frac{\DYN}{D}\right)$. Lemma~\ref{lem:bandit} implies the theorem.
\end{proof}

\section{Implications for covering LPs for \texorpdfstring{$k$}{k}-server}\label{sec:covering}

An \emph{online covering problem} is specified by a linear program of the following form, where the vector $c$ is given upfront, matrix $A_t$ and vector $b_t$ are revealed at time $t$, and the entries of $c$, $A_t$ and $b_t$ are all non-negative:
\begin{align}\label{eq:covering}
    \begin{array}{l@{\quad}rcll}
    \min          & c^\intercal x  \\
    \mathrm{s.t.} & A_t x &\ge& b_t &\forall t=1,\dots,T \\
                  & x   &\in& [0,1]^n&
    \end{array}
\end{align}
We call a variable $x_i$ \emph{active} at time $t$ if the corresponding column of $A_t$ contains a non-zero entry. We denote by $\act(t)$ the indices of active variables (i.e., the indices of columns of $A_t$ that contain a non-zero entry).

An online algorithm for an online covering problem maintains a solution $x$ to the LP that is feasible for the constraints revealed so far and is typically required to be \emph{monotone} in the sense that the value of each variable $x_i$ is non-decreasing over time.

Many online problems can be expressed in this form. Since such formulations lend themselves to the design of randomized algorithms~\citep{BuchbinderN09}, there is significant interest in finding online covering formulations for the $k$-server problem~\citep{BansalBN07,BansalBN10,GuptaKP21}.

An \emph{online covering formulation of the $k$-server problem} is an LP of the form~\eqref{eq:covering} that can be constructed alongside the request arrivals (i.e., $A_t$ and $b_t$ can be specified once the $t$th request is revealed) such that
\begin{itemize}
    \item any monotone algorithm for the online covering problem can be converted \emph{online} into a randomized 
    algorithm for the $k$-server instance whose expected cost is at most a constant factor greater,
    \item any (deterministic) lazy algorithm for the $k$-server instance can be converted \emph{online} into a monotone algorithm for the online covering problem of no greater cost.
\end{itemize}
For a lazy $k$-server algorithm $A$, we denote by $x(t,A)$ the LP solution at time $t$ resulting from the latter conversion.

\begin{definition}
    We call an online covering formulation of a $k$-server instance \emph{configuration-encoding} if, for any two lazy $k$-server algorithms $A$ and $A'$ that are in the same configuration at some time $t$, it holds that $x_i(t,A)=x_i(t,A')$ for all $i\in\act(t)$.
\end{definition}
In other words, an online covering formulation is configuration-encoding if the current server configuration is uniquely determined by the current values of the active variables of the corresponding LP solution.


For example, the weighted paging problem (equivalent to $k$-server on a weighted star metric) is captured by the following configuration-encoding online covering formulation~\citep{BansalBN07}, where $w_p$ is the weight of page $p$, variable $x_{p,r}$ represents the probability that page $p$ is evicted from the cache between the $r$th and $(r+1)$st request to page $p$, $B(t)$ is the set of pages requested up to time $t$, $p_t$ is the page requested at time $t$, $r(p,t)$ is the number of requests to page $p$ up to time $t$, and $k$ is the cache size:
\begin{align*}
    &\min          & \sum_{p}\sum_{r} w_p x_{p,r}&  \\
    &\,\,\mathrm{s.t.} & \sum_{p\in B(t)\setminus\{p_t\}} x_{p,r(p,t)} &\ge |B(t)|-k &&\forall t=1,\dots,T \\
                  && x_{p,r}                                       &\in [0,1]    &&\forall p, r
\end{align*}

Here, each matrix $A_t$ contains only a single row, and the active variables at time $t$ are $x_{p,r(p,t)}$ for $p\in B(t)\setminus\{p_t\}$. The formulation is indeed configuration-encoding, since for the natural conversion of a paging algorithm $A$ into an algorithm $x(\cdot,A)$ for the covering problem, the cache of $A$ at time $t$ contains precisely $p_t$ and the pages $p\in B(t)\setminus\{p_t\}$ with $x_{p,r(p,t)}(t,A)=0$.

The conversion of a monotone algorithm for the LP (which corresponds to a ``fractional'' algorithm) into a randomized algorithm is non-trivial, but known to exist for weighted paging \citep{BansalBN07}. Similarly, for the more general case of $k$-server on HST metrics it is also known how to convert a fractional algorithm into a randomized integral one~\citep{BansalBMN15}. However, we will show now that unlike for weighted paging, no configuration-encoding online covering formulation exists for this more general setting of HSTs:

\begin{proof}[Proof of Theorem~\ref{thm:covering}]
    \citet[Theorem 2.1]{Anand0KP22} considered the version of~\eqref{eq:covering} where each matrix $A_t$ contains a single row and $b_t$ is a number, and at time $t$ the algorithm receives $\ell$ predictions $x(t,1),\dots,x(t,\ell)$ such that each $x(t,s)\in[0,1]^n$ is a feasible solution for the constraints up to time $t$. For this setting, they showed that there is an $O(\log \ell)$-competitive algorithm against the benchmark\footnote{The original version of~\citep{Anand0KP22} states the order of quantifiers as ``$\forall i\,\forall t\,\exists s$'', but this was a typographical error~\citep{Panigrahi23}.}
    \begin{align*}
        \DYNA =\min \{c^\intercal x \mid \forall t\in[T]\,\,\exists s_t\in[\ell] \,\,\forall i\in[n]\colon x_i \ge x_i(t,s_t)\}.
    \end{align*}
    
    The result directly extends to the case where $A_t$ can have several rows and $b_t$ is a vector, as one can simulate their algorithm by revealing the constraints in $A_t x \ge b_t$ row by row. Moreover, inspection of their proof shows that the result holds even against the slightly stronger benchmark that requires $x_i \ge x_i(t,s)$ only for active variables at time $t$:
    \begin{align*}
        \DYNAp =\min \{c^\intercal x \mid \forall t\in[T]\,\,\exists s_t\in[\ell] \,\,\forall i\in\act(t)\colon x_i \ge x_i(t,s_t)\}.
    \end{align*}

    Consider now an instance of the $k$-server problem with full access to $\ell$ predictors for which the lower bound of Theorem~\ref{thm:l^2LB} holds. By the construction in~\citep{Bubeck22kServer} on which the proof of Theorem~\ref{thm:l^2LB} is based, this instance uses a metric space of $n$ points such that $\ell=O(\log n / \log \log n)$. Thus, by~\citep{FakcharoenpholRT04}, the metric space can be probabilistically approximated by an HST with distortion $O(\log n)=O(\ell\log\ell)$. Combined with Theorem~\ref{thm:l^2LB}, this means that even on HSTs there exist instances of $k$-server where no randomized learning-augmented algorithm with full access to $\ell$ predictors has competitive ratio better than $\Omega(\ell^2)/O(\ell\log\ell)=\Omega(\ell/\log\ell)$ against $\DYN$. 

    Now, suppose there exists a configuration-encoding online covering formulation for the $k$-server problem on HSTs. Let $P_1,\dots,P_\ell$ be predictors for the $k$-server instance, and assume without loss of generality that each $P_s$ is lazy (the lower bound still holds for lazy predictors, since it clearly cannot help the algorithm if each predictor defers its movements for as long as possible). We can construct online the corresponding predictions $x(t,s)=x(t, P_s)$ for the online covering problem. We claim that $\DYNAp\le \DYN$. This will suffice to prove the theorem, because then the $O(\log\ell)$-competitive algorithm of~\citet{Anand0KP22} for the online covering problem against $\DYNAp$ can be converted to an $O(\log \ell)$-competitive randomized algorithm for the $k$-server problem against $\DYN$, contradicting the aforementioned lower bound of $\Omega(\ell/\log\ell)$.

    Abusing notation slightly, we write $\DYN$ also for the $k$-server algorithm achieving cost $\DYN$. Since each predictor $P_s$ is lazy, and recalling that the lower bound instance uses a metric space of $k+1$ points, we may also assume that $\DYN$ is lazy.
    Thus, $x(\,\cdot\,,\DYN)$ is a well-defined monotone (offline) algorithm for the covering problem. For $t\in[T]$ and $i\in\act(t)$, we have
    \begin{align*}
        x_i(T,\DYN) \ge x_i(t,\DYN) = x_i(t,s_t),
    \end{align*}
    where $s_t$ is the index of the predictor that $\DYN$ is following at time $t$, and the equation follows since the covering formulation is configuration-encoding. Thus, $x(T,\DYN)$ is a feasible vector for the problem $\DYNAp$ and hence $\DYNAp \le c^\intercal x(T,\DYN) \le \DYN$.
\end{proof}

\bibliography{draft}

\begin{thebibliography}{39}
\providecommand{\natexlab}[1]{#1}
\providecommand{\url}[1]{\texttt{#1}}
\expandafter\ifx\csname urlstyle\endcsname\relax
  \providecommand{\doi}[1]{doi: #1}\else
  \providecommand{\doi}{doi: \begingroup \urlstyle{rm}\Url}\fi

\bibitem[Almanza et~al.(2021)Almanza, Chierichetti, Lattanzi, Panconesi, and
  Re]{AlmanzaCLPR21}
Matteo Almanza, Flavio Chierichetti, Silvio Lattanzi, Alessandro Panconesi, and
  Giuseppe Re.
\newblock Online facility location with multiple advice.
\newblock In \emph{{NeurIPS}}, 2021.

\bibitem[Anand et~al.(2022)Anand, Ge, Kumar, and Panigrahi]{Anand0KP22}
Keerti Anand, Rong Ge, Amit Kumar, and Debmalya Panigrahi.
\newblock Online algorithms with multiple predictions.
\newblock In \emph{{ICML}}, 2022.

\bibitem[Antoniadis et~al.(2020)Antoniadis, Coester, Eli{\'{a}}\v{s}, Polak,
  and Simon]{AntoniadisCE0S20}
Antonios Antoniadis, Christian Coester, Marek Eli{\'{a}}\v{s}, Adam Polak, and
  Bertrand Simon.
\newblock Online metric algorithms with untrusted predictions.
\newblock In \emph{{ICML}}, 2020.

\bibitem[Antoniadis et~al.(2021)Antoniadis, Coester, Eli{\'{a}}\v{s}, Polak,
  and Simon]{AntoniadisCEPS21}
Antonios Antoniadis, Christian Coester, Marek Eli{\'{a}}\v{s}, Adam Polak, and
  Bertrand Simon.
\newblock Learning-augmented dynamic power management with multiple states via
  new ski rental bounds.
\newblock In \emph{{NeurIPS}}, 2021.

\bibitem[Bamas et~al.(2020)Bamas, Maggiori, Rohwedder, and
  Svensson]{BamasMRS20}
{\'{E}}tienne Bamas, Andreas Maggiori, Lars Rohwedder, and Ola Svensson.
\newblock Learning augmented energy minimization via speed scaling.
\newblock In \emph{{NeurIPS}}, 2020.

\bibitem[Bansal et~al.(2007)Bansal, Buchbinder, and Naor]{BansalBN07}
Nikhil Bansal, Niv Buchbinder, and Joseph Naor.
\newblock A primal-dual randomized algorithm for weighted paging.
\newblock In \emph{{FOCS}}, 2007.
\newblock URL \url{https://doi.org/10.1109/FOCS.2007.7}.

\bibitem[Bansal et~al.(2010)Bansal, Buchbinder, and Naor]{BansalBN10}
Nikhil Bansal, Niv Buchbinder, and Joseph Naor.
\newblock Towards the randomized k-server conjecture: {A} primal-dual approach.
\newblock In \emph{{SODA}}, 2010.
\newblock URL \url{https://doi.org/10.1137/1.9781611973075.5}.

\bibitem[Bansal et~al.(2015)Bansal, Buchbinder, Madry, and Naor]{BansalBMN15}
Nikhil Bansal, Niv Buchbinder, Aleksander Madry, and Joseph Naor.
\newblock A polylogarithmic-competitive algorithm for the \emph{k}-server
  problem.
\newblock \emph{J. {ACM}}, 62\penalty0 (5):\penalty0 40:1--40:49, 2015.
\newblock URL \url{https://doi.org/10.1145/2783434}.

\bibitem[Bansal et~al.(2022)Bansal, Coester, Kumar, Purohit, and
  Vee]{BansalCKPV22}
Nikhil Bansal, Christian Coester, Ravi Kumar, Manish Purohit, and Erik Vee.
\newblock Learning-augmented weighted paging.
\newblock In \emph{{SODA}}, 2022.

\bibitem[Bartal et~al.(1997)Bartal, Blum, Burch, and Tomkins]{BBBT97}
Yair Bartal, Avrim Blum, Carl Burch, and Andrew Tomkins.
\newblock A polylog(\emph{n})-competitive algorithm for metrical task systems.
\newblock In \emph{{STOC}}, 1997.
\newblock URL \url{https://doi.org/10.1145/258533.258667}.

\bibitem[Bhaskara et~al.(2020)Bhaskara, Cutkosky, Kumar, and
  Purohit]{BhaskaraC0P20}
Aditya Bhaskara, Ashok Cutkosky, Ravi Kumar, and Manish Purohit.
\newblock Online linear optimization with many hints.
\newblock In \emph{{NeurIPS}}, 2020.

\bibitem[Blum and Burch(2000)]{BlumB00}
Avrim Blum and Carl Burch.
\newblock On-line learning and the metrical task system problem.
\newblock \emph{Mach. Learn.}, 39\penalty0 (1):\penalty0 35--58, 2000.
\newblock \doi{10.1023/A:1007621832648}.

\bibitem[Borodin et~al.(1992)Borodin, Linial, and Saks]{BorodinLS92}
Allan Borodin, Nathan Linial, and Michael~E. Saks.
\newblock An optimal on-line algorithm for metrical task system.
\newblock \emph{J. {ACM}}, 39\penalty0 (4):\penalty0 745--763, 1992.

\bibitem[Bubeck et~al.(2021)Bubeck, Cohen, Lee, and Lee]{BubeckCLL21}
S{\'{e}}bastien Bubeck, Michael~B. Cohen, James~R. Lee, and Yin~Tat Lee.
\newblock Metrical task systems on trees via mirror descent and unfair gluing.
\newblock \emph{{SIAM} J. Comput.}, 50\penalty0 (3):\penalty0 909--923, 2021.
\newblock \doi{10.1137/19M1237879}.
\newblock URL \url{https://doi.org/10.1137/19M1237879}.

\bibitem[Bubeck et~al.(2022{\natexlab{a}})Bubeck, Coester, and
  Rabani]{Bubeck22kServer}
S{\'{e}}bastien Bubeck, Christian Coester, and Yuval Rabani.
\newblock The randomized k-server conjecture is false!
\newblock \emph{CoRR}, abs/2211.05753, 2022{\natexlab{a}}.
\newblock \doi{10.48550/arXiv.2211.05753}.

\bibitem[Bubeck et~al.(2022{\natexlab{b}})Bubeck, Coester, and
  Rabani]{BubeckCR22}
S{\'{e}}bastien Bubeck, Christian Coester, and Yuval Rabani.
\newblock Shortest paths without a map, but with an entropic regularizer.
\newblock In \emph{{FOCS}}, 2022{\natexlab{b}}.

\bibitem[Buchbinder and Naor(2009)]{BuchbinderN09}
Niv Buchbinder and Joseph Naor.
\newblock The design of competitive online algorithms via a primal-dual
  approach.
\newblock \emph{Found. Trends Theor. Comput. Sci.}, 3\penalty0 (2-3):\penalty0
  93--263, 2009.
\newblock URL \url{https://doi.org/10.1561/0400000024}.

\bibitem[Cesa{-}Bianchi and Lugosi(2006)]{BianchiLugosi}
Nicol{\`{o}} Cesa{-}Bianchi and G{\'{a}}bor Lugosi.
\newblock \emph{Prediction, learning, and games}.
\newblock Cambridge University Press, 2006.
\newblock ISBN 978-0-521-84108-5.
\newblock \doi{10.1017/CBO9780511546921}.
\newblock URL \url{https://doi.org/10.1017/CBO9780511546921}.

\bibitem[Chrobak and Larmore(1991)]{ChrobakL91}
Marek Chrobak and Lawrence~L. Larmore.
\newblock The server problem and on-line games.
\newblock In \emph{On-Line Algorithms}, volume~7 of \emph{{DIMACS} Series in
  Discrete Mathematics and Theoretical Computer Science}, pages 11--64.
  {DIMACS/AMS}, 1991.

\bibitem[Chrobak et~al.(1991)Chrobak, Karloff, Payne, and
  Vishwanathan]{chrobak1991dc}
Marek Chrobak, Howard~J. Karloff, T.~H. Payne, and Sundar Vishwanathan.
\newblock New results on server problems.
\newblock \emph{{SIAM} J. Discret. Math.}, 4\penalty0 (2):\penalty0 172--181,
  1991.

\bibitem[Dinitz et~al.(2022)Dinitz, Im, Lavastida, Moseley, and
  Vassilvitskii]{DinitzILMV22}
Michael Dinitz, Sungjin Im, Thomas Lavastida, Benjamin Moseley, and Sergei
  Vassilvitskii.
\newblock Algorithms with prediction portfolios.
\newblock In \emph{{NeurIPS}}, 2022.

\bibitem[Emek et~al.(2009)Emek, Fraigniaud, Korman, and Ros{\'{e}}n]{EmekFKR09}
Yuval Emek, Pierre Fraigniaud, Amos Korman, and Adi Ros{\'{e}}n.
\newblock Online computation with advice.
\newblock In \emph{{ICALP}}, 2009.
\newblock URL \url{https://doi.org/10.1007/978-3-642-02927-1\_36}.

\bibitem[Emek et~al.(2021)Emek, Kutten, and Shi]{EmekKS21}
Yuval Emek, Shay Kutten, and Yangguang Shi.
\newblock Online paging with a vanishing regret.
\newblock In \emph{{ITCS}}, 2021.

\bibitem[Fakcharoenphol et~al.(2004)Fakcharoenphol, Rao, and
  Talwar]{FakcharoenpholRT04}
Jittat Fakcharoenphol, Satish Rao, and Kunal Talwar.
\newblock A tight bound on approximating arbitrary metrics by tree metrics.
\newblock \emph{J. Comput. Syst. Sci.}, 69\penalty0 (3):\penalty0 485--497,
  2004.
\newblock URL \url{https://doi.org/10.1016/j.jcss.2004.04.011}.

\bibitem[Fiat and Mendel(2003)]{FiatM03}
Amos Fiat and Manor Mendel.
\newblock Better algorithms for unfair metrical task systems and applications.
\newblock \emph{{SIAM} J. Comput.}, 32\penalty0 (6):\penalty0 1403--1422, 2003.
\newblock URL \url{https://doi.org/10.1137/S0097539700376159}.

\bibitem[Fiat et~al.(1990)Fiat, Rabani, and Ravid]{FiatRR90}
Amos Fiat, Yuval Rabani, and Yiftach Ravid.
\newblock Competitive k-server algorithms (extended abstract).
\newblock In \emph{{FOCS}}, 1990.

\bibitem[Fiat et~al.(1998)Fiat, Foster, Karloff, Rabani, Ravid, and
  Vishwanathan]{FiatFKRRV98}
Amos Fiat, Dean~P. Foster, Howard~J. Karloff, Yuval Rabani, Yiftach Ravid, and
  Sundar Vishwanathan.
\newblock Competitive algorithms for layered graph traversal.
\newblock \emph{{SIAM} J. Comput.}, 28\penalty0 (2):\penalty0 447--462, 1998.

\bibitem[Gollapudi and Panigrahi(2019)]{GollapudiP19}
Sreenivas Gollapudi and Debmalya Panigrahi.
\newblock Online algorithms for rent-or-buy with expert advice.
\newblock In \emph{{ICML}}, 2019.

\bibitem[Gupta et~al.(2021)Gupta, Kumar, and Panigrahi]{GuptaKP21}
Anupam Gupta, Amit Kumar, and Debmalya Panigrahi.
\newblock A hitting set relaxation for \emph{k}-server and an extension to
  time-windows.
\newblock In \emph{{FOCS}}, 2021.
\newblock URL \url{https://doi.org/10.1109/FOCS52979.2021.00057}.

\bibitem[Herbster and Warmuth(1998)]{herbster1998tracking}
Mark Herbster and Manfred~K Warmuth.
\newblock Tracking the best expert.
\newblock \emph{Machine learning}, 32\penalty0 (2):\penalty0 151--178, 1998.

\bibitem[Im et~al.(2022)Im, Kumar, Petety, and Purohit]{Im0PP22}
Sungjin Im, Ravi Kumar, Aditya Petety, and Manish Purohit.
\newblock Parsimonious learning-augmented caching.
\newblock In \emph{{ICML}}, 2022.

\bibitem[Lykouris and Vassilvitskii(2021)]{LykourisV21}
Thodoris Lykouris and Sergei Vassilvitskii.
\newblock Competitive caching with machine learned advice.
\newblock \emph{J. {ACM}}, 68\penalty0 (4):\penalty0 24:1--24:25, 2021.

\bibitem[Manasse et~al.(1990)Manasse, McGeoch, and Sleator]{manasse1990server}
Mark~S. Manasse, Lyle~A. McGeoch, and Daniel~Dominic Sleator.
\newblock Competitive algorithms for server problems.
\newblock \emph{J. Algorithms}, 11\penalty0 (2):\penalty0 208--230, 1990.

\bibitem[Panigrahi(2023)]{Panigrahi23}
Debmalya Panigrahi.
\newblock Personal communication, 2023.

\bibitem[Purohit et~al.(2018)Purohit, Svitkina, and Kumar]{PurohitSK18}
Manish Purohit, Zoya Svitkina, and Ravi Kumar.
\newblock Improving online algorithms via {ML} predictions.
\newblock In \emph{{NeurIPS}}, 2018.

\bibitem[Rohatgi(2020)]{Rohatgi20}
Dhruv Rohatgi.
\newblock Near-optimal bounds for online caching with machine learned advice.
\newblock In \emph{{SODA}}, 2020.

\bibitem[Slivkins(2019)]{Slivkins19}
Aleksandrs Slivkins.
\newblock Introduction to multi-armed bandits.
\newblock \emph{CoRR}, abs/1904.07272, 2019.
\newblock URL \url{http://arxiv.org/abs/1904.07272}.

\bibitem[Wang et~al.(2020)Wang, Li, and Wang]{WangLW20}
Shufan Wang, Jian Li, and Shiqiang Wang.
\newblock Online algorithms for multi-shop ski rental with machine learned
  advice.
\newblock In \emph{{NeurIPS}}, 2020.

\bibitem[Wei(2020)]{Wei20}
Alexander Wei.
\newblock Better and simpler learning-augmented online caching.
\newblock In \emph{{APPROX/RANDOM}}, 2020.

\end{thebibliography}
\bibliographystyle{plainnat}

\newpage
\appendix
\onecolumn

\section{Algorithms for unfair MTS}
\label{app:algs}

\subsection{Odd Exponent}
\label{app:oddexp}
There is a randomized algorithm \OddExponent based on work functions
proposed by \cite{BBBT97}.

Choose an odd number $a$ close to $\ln \ell$ as a parameter.
State $j$ is chosen with probability
\[ p_j :=  \frac1\ell
+ \frac1\ell \sum_{i=1}^\ell (w_t(i)- w_t(j))^a, \]
where $w_t(i)$ is the work function of state $i$ at time $t$:
\begin{align*}
w_t(i) = \min\{ &d(i, x_t) +\\ 
& \sum_{j=1}^t (c_j(x_j) + r\cdot d(x_{j-1},x_j))
\,|\; x_j \in M \}.
\end{align*}
In order to be well defined (e.g., for all probabilities to be non-negative),
the input sequence $c_1, \dotsc, c_T$ needs to satisfy the following properties:
\begin{itemize}
\item Each $c_t$ has only one non-zero coordinate
\item If $c_t$ with non-zero value $c_t(i)$ would make \OddExponent
	remove all probability mass from state $i$,
	we assume that $c_t(i)$ is the smallest such value.
\item If \OddExponent already has 0 probability mass at state $i$,
	no cost function with $c_t(i)>0$ arrive.
\end{itemize}
These properties can be assumed without loss of generality, since
they can be achieved by splitting each cost function into several smaller ones
and omitting those which do not imply any cost on the algorithm
(this omission does not increase the cost of the offline optimum either).
We refer to Section 4.4.1 in \cite{BlumB00} for more details on
how to implement this algorithm in the general MTS setting.

\subsection{\share}
\label{app:share}
\share is an algorithm for tracking the best expert regime in Online Learning
proposed by \citet{herbster1998tracking}.
It requires two parameters:
the sharing parameter $\alpha\in[0,1/2]$ and $\beta\in [0,1]$ (logarithm
of the learning rate).

We can apply it to unfair MTS in uniform metric space of size $\ell$ with cost
functions bounded by $1$ as follows. It starts with weights $w_0(1) = \dotsb
w_0(\ell)=1$ and uniform probability distribution over the states, i.e.,
$p_0(i) = w_0(i) / \sum_{j=1}^\ell w_0(j)$.
At time $t$, when its probability distribution over states is $p_t$,
it incurs cost $\langle p_t, c_t\rangle$ and updates the weights
and its probability distribution for the next time step:
\begin{align*}
w_{t+1}(i) &:= w_t(i) \cdot \beta^{c_t(i)} + \alpha \Delta/\ell\\
p_{t+1}(i) &:= w_{t+1}(i)/\sum_{j=1}^\ell w_{t+1}(j),
\end{align*}
where $\Delta = \sum_{i=1}^\ell (w_t(i) - w_t(i)\beta^{c_t(i)})$.
This way, its distribution $p_t$ depends only on $c_1, \dotsc, c_{t-1}$
and proposes only one distribution $p_t$ at each time step.
Proposition~\ref{prop:share} by \citet{BlumB00} states the performance
guarantee of this algorithm for unfair MTS.
Note that this algorithm can be easily adapted to unbounded cost functions:
we split each cost function into several smaller cost functions
bounded by $1$. Due to this splitting, however, it may move several
times during each time step.

\end{document}